\pdfoutput=1
\documentclass{article}
\usepackage[a4paper, total={6in, 8in}]{geometry}
\usepackage[utf8]{inputenc}
\usepackage{amsmath}
\usepackage{amssymb}
\usepackage{algorithm}
\usepackage{algorithmicx}
\usepackage{algpseudocode}
\usepackage{multirow}
\usepackage{subcaption}
\usepackage{xcolor}
\usepackage{comment}
\usepackage{breqn}
\usepackage{hyperref}
\usepackage{todonotes} 
\usepackage{color}
\usepackage{soul}
\usepackage{amsthm}
\usepackage{biblatex}
\bibliography{ms}
\usepackage[english]{babel}
\newtheorem{definition}{Definition}
\newtheorem{theorem}{Theorem}

\AtBeginDocument{%
  \providecommand\BibTeX{{%
    \normalfont B\kern-0.5em{\scshape i\kern-0.25em b}\kern-0.8em\TeX}}}

\begin{document}
\title{Inferential SIR\_GN: Scalable Graph Representation Learning}



\author{Janet Layne,  Boise State University, janetlayne@boisestate.edu\\ \ Edoardo Serra, Boise State University, edoardoserra@boisestate.edu}


\date{May 2021}
\maketitle
\section*{}
    Graph representation learning methods generate numerical vector representations for the nodes in a network, thereby enabling their use in standard machine learning models. These methods aim to preserve relational information, such that nodes that are similar in the graph are found close to one another in the representation space. Similarity can be based largely on one of two notions: connectivity or structural role. In tasks where node structural role is important, connectivity based methods show poor performance. Recent work has begun to focus on scalability of learning methods to massive graphs of millions to billions of nodes and edges. Many unsupervised node representation learning algorithms are incapable of scaling to large graphs, and are unable to generate node representations for unseen nodes. In this work, we propose Inferential SIR-GN, a model which is pre-trained on random graphs, then computes node representations rapidly, including for very large networks. We demonstrate that the model is able to capture node's structural role information, and show excellent performance at node and graph classification tasks, on unseen networks. Additionally, we observe the scalability of Inferential SIR-GN is comparable to the fastest current approaches for massive graphs.

\section{Introduction}

Graphs are universal data structures for capturing relational information about a dataset in a very simple arrangement: a set of nodes representing each entity, and edges representing connections or relations between entities. This data structure is easy to access and perform simple calculations on the graph. However, modern tasks to utilize graph data, such as node classification, link or edge prediction, community detection, clustering, graph classification etc. rely on standard Machine Learning (ML) approaches. These techniques are not able to operate on graph data in its original form; the graph must be transformed into a vector of numerical features that enable discrimination between different entities. It is an ongoing challenge to project graph data into a feature space in a manner that preserves the relational structure of the original network.  Hand-engineered features perform well, but are time-prohibitive, and can require extensive domain expertise. Additionally, features specific to one graph are unlikely to generalize well to new graphs, possibly even within the same domain.\par
Graph representation learning refers to a body of approaches that automatically generates a numerical vector representation for each node from a list of nodes and edges. Domain agnostic, these methods use only the information present in the graph itself. The goal of these methods is to project the node into some low-dimensional latent space, while maintaining that similar nodes in the graph will be found close to one another in said space. Importantly, the notion of similarity is not fixed; approaches can largely be understood as either capturing node proximity or structural properties in their concept of similarity.\par
 Proximity-based similarity will generate similar node representations for nodes that have short paths between them, and many paths connecting them. Nodes close to one another and with many neighbors in common tend to be projected close to one another in the latent feature space. In Figure \ref{fig:intro} shows a simplified visual description of the difference between node connectivity-based similarity,and structure-based similarity. Consider that the two subgraphs formed by nodes $a,b,c$ and $d,f,g$ are structures distant from one another in the same graph. Similarity based upon node connectivity will create similar node representations for nodes $a$ and $b$, and for nodes $d$ and $f$, but $a$ and $d$ will be considered very different. However, when node structural roles are considered, nodes $a$ and $d$ will be considered identical even if no path that connects them. Either methodology is useful in certain circumstances: as an example, link prediction is used to predict whether two entities will be connected. Methods which preserve node connectivity are likely to perform well at such tasks, as two nodes with many neighbors in common are also likely connected. For node classification, let us consider example of a social network, where we want to predict the moderator of a social group. In this instance, the node structural information is a better predictor of the functional role of an entity. Previous works \cite{joaristi2021sir} have demonstrated that node representations which preserve structural information have superior performance in several tasks.

\begin{figure*}[t]
    \center
    \begin{subfigure}[b]{0.8\textwidth}{\includegraphics[width=\linewidth]{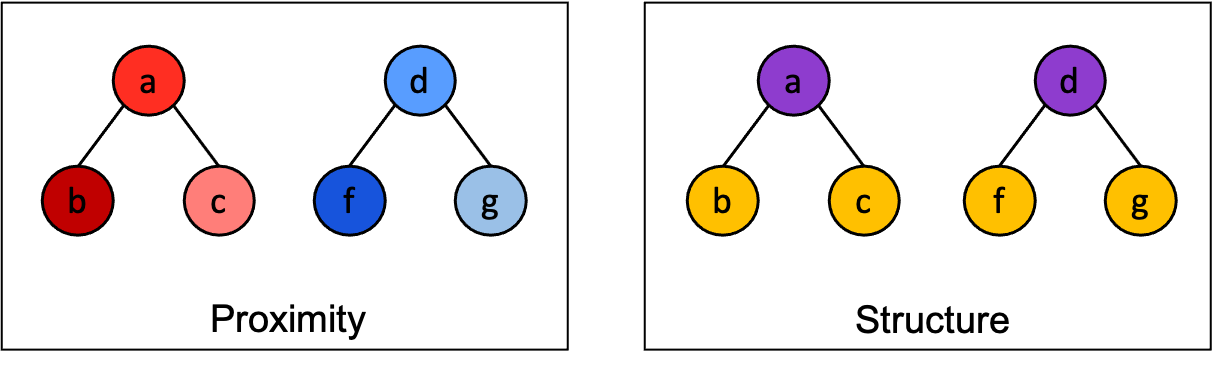}}
    \end{subfigure}
    \caption{Simplified schematic of node similarity. Node color indicates similarity of node representation}\label{fig:intro}
\end{figure*}

For many graph representation learning methods, the time complexity makes scaling to massive graphs prohibitive. A great deal of recent work has attempted to generate algorithms that can scale to graphs with nodes and edges into the millions to billions. Much of this work focuses on random  walk based approaches, or factorization of a proximity matrix. Materializing the proximity matrix, however, can be prohibitively time consuming, so methods to improve scalability include approximating the matrix to be factorized.

A consistent disadvantage with many unsupervised graph representation learning techniques is their inability to generate node representations for unseen nodes. The entire graph must be embedded at one time, and the operation performed from scratch if the graph changes (nodes or edges are added or removed). As these methods have high training overheads, frequently updating dynamic graphs becomes computationally prohibitively expensive. Graph Neural Network based approaches, such as Graph Convolutional Networks, and Graph Isomorphism Networks, can extend learned models to embed unseen nodes, but often show poor performance due to their tendency to overfit.

We propose an improved Inferential SIR-GN, which uncouples training of the model from generation of node representations, increasing both scalability and capacity for generalization. Inferential SIR-GN trains a model on a large number of random graphs, with user-designated hyperparameters such as graph size, depth of exploration, and node representation size. As a result of upstream training, the time to generate node representations is drastically reduced, making our model easily scalable to graphs of any size. We will show that even models trained with hyperparameters that minimize training time have the ability to generate node embeddings that effectively capture the structure of large, complex graphs. Additionally, the same model can be used to rapidly create node and graph representations for widely variable datasets with respect to graph characteristics, and still produce excellent performance in downstream ML tasks. If desired, models can be more tailored to produce optimized performance, but we will show that minimal tailoring is usually required. Our model includes the capability to include both node and edge features, and can be trained on a variable number of features without losing performance.
The Inferential SIR-GN algorithm also includes a novel method for generating graph representations that improves graph classification tasks. We propose a structural pseudo-adjacency matrix that represents each graph, but is agnostic to node order, unlike standard adjacency matrices.
The remainder of this article will present related work, followed by descriptions of the Inferential SIR-GN algorithms and its modifications. Finally, experiments will be described and discussed that indicate the performance of Inferential SIR-GN for node and graph classification. We will show that Inferential SIR-GN can generate node representations on massive graphs as efficiently as the most rapid existing methods, and that those representations accurately capture node's structural information.

\section{Related Works}
In \cite{joaristi2021sir}, SIR-GN performed comparisons with node2vec\cite{grover2016node2vec}, struc2vec\cite{ribeiro2017struc2vec}, GraphWave\cite{donnat2018learning}, GraphSAGE\cite{hamilton2017inductive}, ARGA\cite{pan2018adversarially} and DRNE\cite{tu2018deep}. As the original SIR-GN algorithm performed as well or better than these at several tasks, we will restrict our comparisons to SIR-GN, along with algorithms which are designed to scale to very large graphs, as is Inferential SIR-GN.\par
VERSE \cite{verse2018} uses a single-layer neural network to learn similarity distributions from any similarity measure, with Personalized PageRank as default. Multiple methods factorize a proximity matrix to obtain node representations. To accomplish scalability to massive graphs, these methods aim to avoid materializing the proximity matrix. AROPE \cite{arope2018} is able to acquire any desired order proximity matrix by performing eigen-decomposition of the adjacency matrix. They show that a simple re-weighting can shift between different-order proximities. AROPE, however, can only operate on undirected graphs, as the adjacency matrix must be symmetric. In \cite{netmf2017}, authors showed that many random walk network embedding methods are implicitly factorizing a matrix, and the explicit factorization can create more powerful embeddings. However, because the matrix is dense, this is a costly endeavor and not scalable for massive graphs. In subsequent work \cite{netsmf2019}, these same authors propose NetSMF, which sparsifies the DeepWalk\cite{perozzi2014deepwalk} matrix using the spectral graph sparsification technique, followed by randomized singular value decomposition to factor the sparsified matrix. ProNE also uses matrix factorization, but with spectral propagation as an enhancement \cite{prone2019}. STRAP \cite{strap2019}factorizes the transpose proximity matrix M  = $\Pi + \Pi^T$, where $\Pi$ represents the approximate Personalized PageRank(PPR) matrix of graph $G$, and $\Pi^T$ is its transpose. NRP \cite{nrp2020} details the limitation of PPR proximities as relative to a source node, and proposes a method which also factorizes an approximated PPR matrix, followed node re-weighting based upon node degree. \par
We will also consider two graph neural network methods expected to show good performance with node structural classification tasks: graph convolution network\cite{kipf2016semi} and a graph isomorphism network \cite{xu2018powerful}. Both employ a method based upon the Weisfeiler-Lehman's graph isomorphism test, wherein a node's structural representation is updated via neighbor aggregation. For the GCN, aa mean aggregation function is used, while GIN relies upon a sum aggregation function for pooling node representations.

\section{Background}
Here we present concepts and definitions to be referred to throughout this work. We will then define the structural graph metrics that will be used to demonstrate the proficiency of Inferential SIR-GN at capturing node structural information. Table \ref{tab:notations} contains notations used throughout this work.\par

\begin{definition}
\textbf{Graph}: Let $G = (V,E)$ be a unlabeled undirected graph, or network, where $V$ denotes the set of vertices, or nodes, and $E\subseteq(V \times V)$ denotes the set of edges connecting the nodes in $G$. $|V|$ and $|E|$ denote the number of nodes and edges contained in the graph, respectively. The neighborhood of any node $u$ is given by $N(u) = \{v|(u,v)\in E $ $\:\:\:\: or \:\:\:\: (v,u)\in E\}$, and $|N(u)|$ denotes the size of this neighborhood or the node's degree.
\end{definition}

The Betweenness Centrality (BC) \cite{freeman1977set} of a node is the count of the number of times that node acts as a bridge along the shortest path between two other nodes. For every node pair that does not include $v$, the shortest path between them is determined, and the betweenness centrality for a node is the number of shortest paths that pass through $v$
 Formally the Betweenness Centrality  of node $v$ is defined as

\begin{equation} 
BC(v) = \sum_{s \ne v \ne t \in V}\frac{\sigma_{st}(v)}{\sigma_{st}},
\end{equation} 
where $\sigma_{st}(v)$ is the number o shortest paths from vertex $s$ to vertex $t$ paths that passes through vertex $v$ and $\sigma_{st}$ is the total number of shortest paths going from $s$ to $t$.

The Degree centrality (DC) \cite{freeman1978centrality} of a node $v$ is the fraction of the total nodes in the graph to which it is connected. This is mathematically defined as

\begin{equation} 
DC(v) = \frac{|N(v)|}{|V|}.
\end{equation} 

The Eigenvector Centrality (EC) \cite{bonacich1972factoring} is a measure of the influence of a node in a graph, based upon the influence of the nodes to which it is connected. Let $A = (a_{v,t})$ be the adjacency matrix where $a_{v,t} = 1$ if vertex $v$ is linked to vertex $t$, and $a_{v,t} = 0$ otherwise and $\lambda$ is a constant. The Eigenvector Centrality $x_v$ of node $v$ is defined as

\begin{equation} 
x_v = \frac{1}{\lambda}\sum_{t \in V}a_{v,t}x_t.
\end{equation}

Rearranged and rewritten in vector notation, this becomes the familiar eigenvector equation:
\begin{equation} 
Ax=\lambda x.
\end{equation}
The $vth$ component of the dominant eigenvector of the adjacency matrix is the eigenvector centrality of node $v$.

The PageRank (PR) \cite{brin1998anatomy} metric is a variation of the Eigenvector Centrality, originally developed to rank web pages for search engines. The PageRank of node $v$ is defined as

\begin{equation} 
PR(v) = \frac{1-\delta}{|V|}+\delta \sum_{u \in nbr(v)} \frac{PR(u)}{|\{v'|(u,v') \in E \}|},
\end{equation} 
where $\delta$ is a "damping factor", the probability of a node reaching another node via a walk, as opposed to being teleported to the other node.

The HITS algorithm \cite{kleinberg1999authoritative} (Hyperlink-Induced Topic Search) is a two-score link analysis algorithm originally designed to rate Web page relevance. A hub is a node that points to many other nodes, an authority is a node that is pointed to by many nodes. In a directed graph, the authority quantifies the value of the node based upon the total score of the hubs that point to it. The second score is the hub, which is based upon the total score of the authorities to which it points. In the case of an undirected graph, hub and authority scores for a node are the same.

The Node Clique Number (NCN) for a node is the size of the largest Maximal Clique containing the given node. The Maximal Clique\cite{ouyang1997dna} is defined as follows: given a set of nodes, some of which have edges in between them, the maximal clique is the largest subset of nodes in which every node is connected to every other in the subset.

\begin{table}[t]
\centering
\caption{Notations found throughout this work.}\label{tab:notations}
\begin{tabular}{cc}
\hline
Notation        & Description                  \\
\hline
$G$        &  given graph                   \\ 
$V$        &  set of nodes in the graph     \\
$E$        &  set of edges in the graph     \\
$|V|$      &  number of vertices            \\
$|E|$      &  number of edges               \\
$u$        & single node in the graph          \\
$N(u)$   &  set containing $u$'s neighbors \\
$|N(u)|$ &  degree of $u$                 \\
\hline
\end{tabular}%

\end{table}

\begin{table}[t]
\centering
\caption{Notations used in Algorithm 1.}\label{tab:notations:2}
\begin{tabular}{cc}
\hline
Notation        & Description                  \\
\hline
$m$        & The number of nodes chosen for the synthetic graph  \\ 
$g$        & The number of synthetic graphs chosen for training \\
$w$        & The number of clusters for the KMeans for graph representation \\
$c$        & The number of clusters for KMeans for node representation \\
$p$       & The number of PCA components\\
$d$        & The chosen depth to explore (hops)     \\
$deg$        & Array of node degree    \\
        \\ \hline
\end{tabular}%
\end{table}

\section{Methodology}
\subsection{General Description of the Method}
\begin{figure}[h!]
    \center
    \begin{subfigure}[b]{0.8\textwidth}{\includegraphics[width=\linewidth]{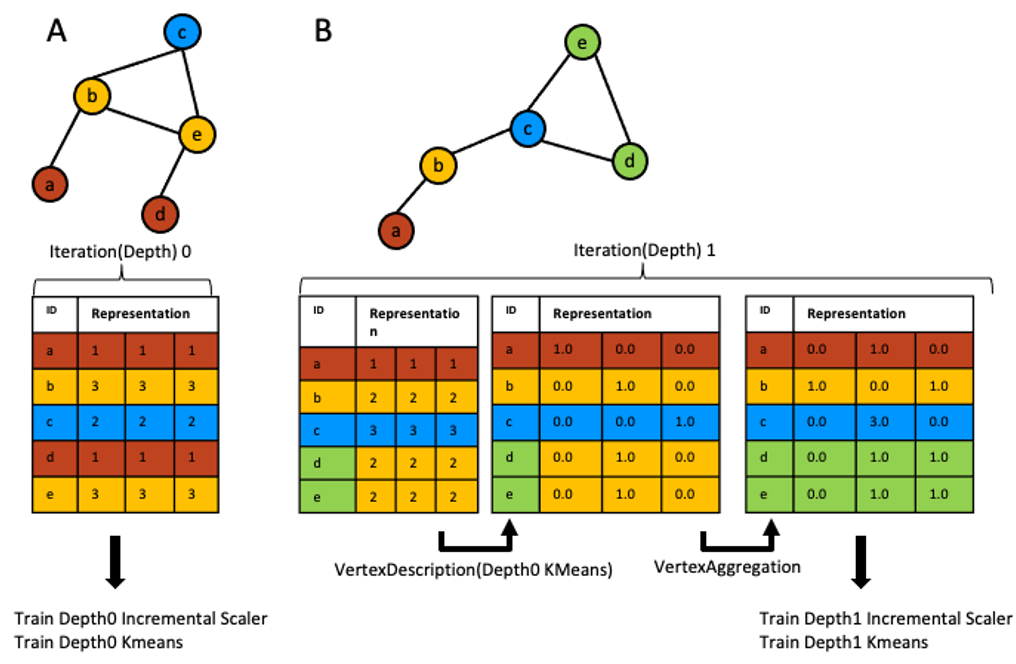}}
    \end{subfigure}
    \caption{Schematic of training Inferential SIR-GN. Graphs depicted represent simplified examples of random graphs that will be used to train the model and node representations show the result of each step of processing during training.}\label{fig:train:1}
\end{figure}
\begin{figure}[h!]
    \center
    \begin{subfigure}[b]{0.8\textwidth}{\includegraphics[width=\linewidth]{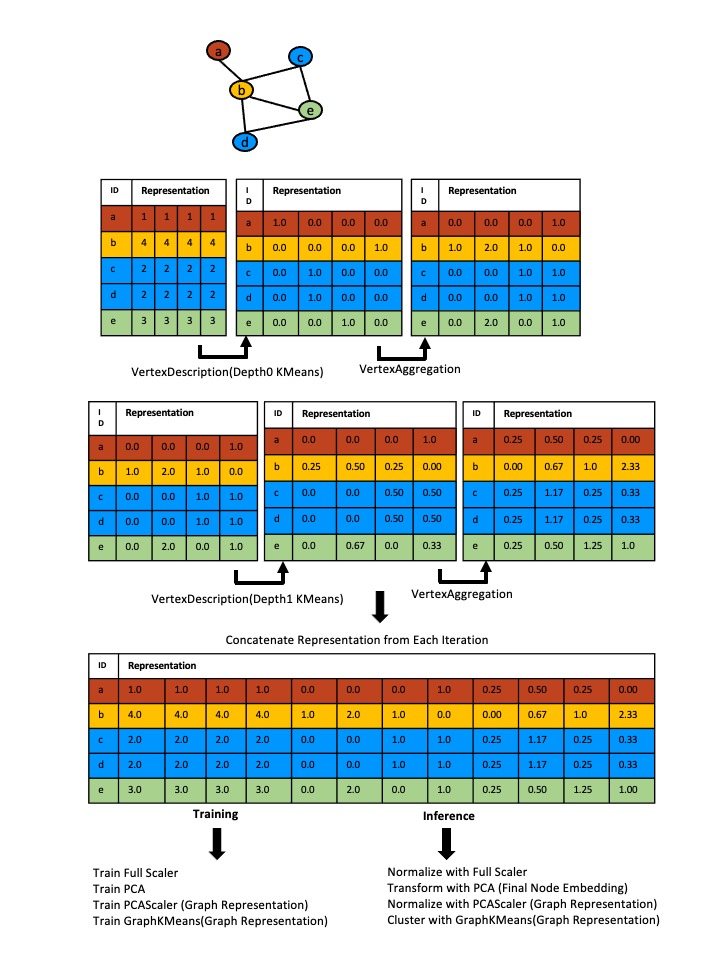}}
    \end{subfigure}
    \caption{Schematic of training Inferential SIR-GN. Graphs depicted represent simplified examples of random graphs used to train the model, and node representations show the result of each step of processing during training.}\label{fig:train:2}
\end{figure}
Inferential SIR-GN uses the same principle of node representation as that for the unmodified SIR-GN described in \cite{joaristi2021sir}. Briefly, each node's representation begins as that node's degree, then is updated iteratively by repeated clustering followed by aggregation of its neighbors. The previous SIR-GN model trained the KMeans and Scaler for every level of iteration using the actual dataset at the time that the node representation was computed. Our proposed method trains the KMeans offline on random graph data that is first processed to emulate the state of the target graph node representations at each level of iteration. For instance, the KMeans cluster centroids that will be used to cluster node representations at inference time for the 2-hop exploration of the graph will be trained on random graphs that have undergone two iterations of clustering followed by aggregation. Thus, the training data, though random, will be representative of the appearance of unseen graphs. This same method is also used to train scikit StandardScaler models for inference.
Another addition to this proposed method does not simply return the final node embedding at the last depth of exploration, but concatenates the embedding from each iteration. This produces a larger node representation that captures the evolution of the node's structural description as its neighborhood is explored. A PCA is trained on random graph data that has been fully processed and concatenated and so emulates the final node embedding that will result from inference. This pre-fitted PCA condenses the final embedding and prevents degradation of information as the node representation size increases.  
Inference is performed using the pre-trained KMeans clusters, StandardScalers, and PCA models. The result of of the process is a condensed structural node representation based upon the exploration of each node's neighborhood to a desired depth. Pre-training of the KMeans, PCA, and StandardScalers are the driving force for the inferential ability of the modified model we present. Because the fitting steps for the KMeans, Scalers, and PCA are removed from computation of the node representation, the time to generate node representations is vastly decreased. We will experimentally demonstrate that pre-training on random graphs still yields a model with excellent performance and efficiency.

The detailed description of the Inferential SIR-GN methodology will be broken into $training$ (Alg. \ref{alg:nodefit}) and $inference$ (Alg. \ref{alg:transform}) algorithms. Both rely upon the internal node clustering and aggregation that is highly similar to that carefully outlined in Joaristi and Serra. The methodology for training the KMeans cluster centroids, StandardScaler models, and PCA are highly redundant in terms of the node description clustering and aggregation. Therefore, we present a detailed algorithm for training the KMeans cluster centroids, and describe the training for the other models in the text, including any differences.  

Inferential SIR-GN has additionally been optimized to generate node representations for  directed graphs, along with both node and edge-labeled data. In addition, we created a model to produce graph representations. Each of these algorithms differ in only small ways, so the algorithm for undirected graphs will be described in detail, then differences from that model will be highlighted in our descriptions of the modifications. Full pseudocode descriptions for each can be found in the Appendix.

\subsection{Inferential SIR-GN for Undirected Graphs}

\subsubsection{Train for Undirected Graphs}
\begin{algorithm}[h!]
\caption{\textsf{Train Inferential SIR-GN UnGraph Node Representations} algorithm}\label{alg:nodefit}
\begin{algorithmic}[1]
\Function{TrainInferentialSIR-GN}{$m, c, p, d, g$}
    \State $KMeansPerLevel = \Call {TrainKMeans}{g, m, d, c}$\label{alg:nodefit:1}
    \State $ScalerPerLevel = \Call {TrainScaler}{g, m, d, c}$
    \State $ScalerFull = \Call {TrainScalerFull}{g, m, d, c}$
    \State $PCA = \Call{TrainPCA}{g, m, p, KMeansPerLevel}$\label{alg:nodefit:2}
    \State \Return $KMeansPerLevel, ScalerPerLevel, ScalerFull, PCA$\label{alg:nodefit:3}
\EndFunction
\Statex
\Function{TrainKMeans}{$g, m, d, c$}\label{alg:nodefit:4}
    \State $KMeansPerLevel = [ ]$
    \State $i = 0$
    \State $j=0$
    \While{$i <= d$}\label{alg:nodefit:19}
        \While{$j < g$}\label{alg:nodefit:5}
            \State $G1 = generateRandomGraph (m)$\label{alg:nodefit:6}
            \State Initialize a matrix \textit{$Emb \in \mathbb{Z}^{|V|\times c}$} to ones \label{alg:nodefit:7}
            \State $Emb = Emb \times deg$ \label{alg:nodefit:8}
            \ForAll{$KMeans$ in  $KMeansPerLevel$}\label{alg:nodefit:9}
                \State $Emb1 = Norm(Emb)$ \label{alg:nodefit:21}
                \ForAll{$u \in V$} \Comment{Vertex description loop}  \label{alg:nodefit:20}
                    \State $dV_u = \textit{CalcDist}(Emb1_u,KMeans)$ \label{alg:nodefit:10}
                    \State $DV_u = (\textit{Max}(dV_u)-dV_u)/(\textit{Max}(dV_u)-\textit{Min}(dV_u))$\label{alg:nodefit:11}
                    \State $DV_u = DV_u/\textit{Sum}(DV_u)$\label{alg:nodefit:12}
                \EndFor
                \ForAll{$u \in V$} \Comment{Aggregation loop} \label{alg:nodefit:13}
                    \State $CR_u = [0,..,0]$\label{alg:nodefit:14}
                    \ForAll{$n \in nbr(u)$} \label{alg:nodefit:15}
                        \State $CR_u = CR_u + DV_n$\label{alg:nodefit:16}
                    \EndFor
                \EndFor \label{alg:nodefit:23}
            \State $Emb = CR$\label{alg:nodefit:22}
            \EndFor
            \State $KMeans = MiniBatchKMeans.partial\_fit(Emb, c)$\label{alg:nodefit:17}
        \EndWhile
        \State append $KMeans$ to $KMeanPerLevel$\label{alg:nodefit:18}
    \EndWhile
    \State \Return $KMeanPerLevel$
 \EndFunction 
 \Statex
 \end{algorithmic}
\end{algorithm}

The training step of Inferential SIR-GN is generally accomplished by computing node representations in the same manner as that for inference, but where new random graphs are generated as input at every iteration, and used to train the KMeans, StandardScalers and PCA. Training of the KMeans cluster centroids and StandardScalers for later inference is accomplished in batches by repeating seven main steps:
\begin{itemize}
\item \textbf{(1)} Generate a random graph
\item \textbf{(2)} Initiate the node representation for each node as its degree 
\item \textbf{(3)} Use fitted K Means cluster centroids from previous iterations to create node descriptions of the graph for the current iteration
\item \textbf{(4)} Repeat steps 1-3, each time using a partial fit to train the StandardScaler
\item \textbf{(5)} Repeat steps 1-3, each time using a partial fit to train the MiniBatchKMeans
\item \textbf{(6)} Store the new KMeans cluster centroids and trained Scaler for use in subsequent iterations and inference.
\item \textbf{(7)} Repeat steps 1-6 until a desired depth of node exploration has been reached. The number of iterations chosen here will correspond to the k-hop neighborhood explored for each node.
\end{itemize}

To initialize, the user selects multiple hyperparameters for training Inferential SIR-GN. First, the user selects the number of random graphs that will be used for training ($g$), and the size of each random graphs ($m$) is input. The desired depth of exploration ($d$) is declared, and represents the number of iterations that will be performed, and the number of separate KMeans that will be fitted to generate node representations.
The user chooses a number of clusters ($c$) whose centroids will be fitted to the random graphs. We will demonstrate that a wide range of values are easily tolerated, and still produce good results. If a graph representation is to be generated (optional), a number of clusters ($w$) for a separate KMeans is indicated. The user also selects the number of PCA components ($p$) to which the final representation will be condensed. This number determines the final node representation size that is returned from Inferential SIR-GN.  
To begin, we generate a random graph (Alg \ref{alg:nodefit}, Line \ref{alg:nodefit:6}) in step (1). Our graph generator uses a random number to determine an edge percentage based upon the input number of nodes. Edges are added randomly between nodes until that edge percentage is reached. The initial representation for each node is set to a vector of size $c$ with all dimensions set as the node degree (Line \ref{alg:nodefit:7} -\ref{alg:nodefit:8}). When the list of trained KMeans is empty (as in depth/iteration zero), the Vertex Description and Vertex Aggregation loops are bypassed, and the node representation remains as the node degree (Line \ref{alg:nodefit:9}. Figure \ref{fig:train:1} shows a simple visual description of the training process for Inferential SIR-GN for node representation. In panel A an example graph is initialized using the node degree with three clusters. The KMeans cluster centroids for depth zero are fitted to a number of these random graph representation chosen by the user (Line \ref{alg:nodefit:17}). For simplicity, Figure \ref{fig:train:1} only shows a single random graph representing a set, however, an entirely separate set of random graphs is generated then processed in this manner independently to train a StandardScaler (a MinMaxScaler can be substituted) for depth zero. This $depth0$ scaler is used to transform the node representations generated to train $depth0$ KMeans cluster centroids. Importantly, training on different sets of random graphs enhances the capability of Inferential SIR-GN to generalize to a large variety of unseen graphs. The trained KMeans and StandardScaler are stored in a list for use in subsequent iterations and later inference (Line \ref{alg:nodefit:18}).
For the second (and any subsequent) layer of depth, new random graphs are generated, and the node representations are calculated using the KMeans trained in the previous iteration. Figure \ref{fig:train:1}B gives an example of this process, where nodes from a random graph (representing a training set) are first represented again by their degree. At this stage, however, a trained KMeans has been stored from the previous iteration, and so we use the cluster centroids to generate a vertex description for each node. Each dimension in the node representation represents a cluster, and the vertex description loop simply transforms the node degree vector into a vector of probabilities that the node is a member of each cluster. First, the node degree vector is normalized using a StandardScaler (Line \ref{alg:nodefit:21}). Following normalization, the vertex description loop is used to transform the representation for each node (Line \ref{alg:nodefit:20}). The loop converts the normalized node degree vector into a probability for each cluster that said node is a member of that cluster. First, the distance from each cluster centroid is computed, resulting in a distance vector of the node degree to each cluster(Line \ref{alg:nodefit:10}). Then each distance is subtracted from the maximum distance in the vector, and this difference is divided by the size of the range of distances for the vector. This results in a vector where the values are inversely proportional to the distance of the degree from the cluster(Line \ref{alg:nodefit:11}). Finally, this vector is normalized by the sum of its values, so that it totals to one (Line \ref{alg:nodefit:12}). 
After the vertex description loop, each node representation is a vector corresponding to its probability of membership in each cluster. Figure \ref{fig:train:1}B depicts the outcome of the vertex description loop when the number of clusters is equal to the number of different node degrees in the graph. This is unlikely in real world situations, even with random graphs, but lends ease of understanding. Here, after a single iteration, each node is fully a member of a cluster. 
The vertex aggregation loop (Line \ref{alg:nodefit:13}) sums for each cluster the probabilities associated with all of a node's neighbors (Line \ref{alg:nodefit:15}-\ref{alg:nodefit:16}). This results in an expected count of each neighbor per cluster as a node's final aggregated representation (Fig. \ref{fig:train:1}B). The embedding generated from the vertex description and aggregation loops is then normalized using the previously trained scaler for the current depth, then used to fit the next iteration of the KMeans cluster centroids (Line \ref{alg:nodefit:17}, Fig. \ref{fig:train:1}B). Early in the exploration, all nodes with the same degree (Fig. \ref{fig:train:1}B, nodes $b, c, d$) have the same representation. However, after the aggregation of the neighbors, different node representations result based on the number of neighbors with each degree. In this way, each iteration $k$ of clustering followed by aggregation creates node representations based upon exploration of the k-hop neighborhood of each node.
This continues successively to the user-designated depth, generating a set of scalers and cluster centroids where, at level $k$, both are fitted for the expected types of values that will result from $k$ rounds of clustering and aggregation. Figure \ref{fig:train:2} shows the result of exploration to a depth 2, where both the $depth0$ and $depth1$ KMeans cluster centroids and are used to generate a node representation. Normalization is left out for simplicity. Here, as in Figure \ref{fig:train:1}, nodes are fully discriminated after only a single round of clustering and aggregation of neighbors. For more complex graphs, usually a greater level of discrimination is reached at increased depth. However, even with larger graphs, nodes are well discriminated quickly, and we will show that increasing depth (and therefore training and inference time) is not necessary beyond a certain level. Importantly, nodes with the exact same structure (nodes $c$ and $d$ in Fig. \ref{fig:train:2}) result after any number of iterations. This was established in \cite{joaristi2021sir} via both proofs and experimentation, and as the internal vertex description and aggregation are identical here, we can safely assume those proofs hold.
As mentioned above, for inference, the structural representation for each iteration is concatenated to create a larger final node representation, which is normalized then condensed using a PCA. To train this PCA, an entirely separate set of random graphs is generated, and the pre-trained KMeans cluster centroids are used to compute iterative node descriptions in a manner identical to that depicted in Algorithm \ref{alg:nodefit}. The distinctive aspect of this training step is that each node representation for each depth is concatenated to previous iterations, so the node description grows with each depth level by a size $c$. Figure \ref{fig:train:2} shows an example final representation after a 2-hop neighborhood exploration. It is clear from the concatenated representation how information can be captured about the evolution of the node's structural representation by storing the result of each level of exploration. This full-sized embedding is used to train a StandardScaler, here referred to as FullScaler. To preserve the capability of Inferential SIR-GN to generalize to unseen data, an independent set of random graphs is processed (via the same concatenation-based methodology shown in Fig. \ref{fig:train:2}), and normalized using the FullScaler. The PCA is then trained on this data. This algorithm returns a model which consists of:
\begin{itemize}
\item \textbf{(1)} A set of KMeans cluster centroids, each corresponding to a depth of neighborhood exploration
\item \textbf{(2)} A set of fitted StandardScalers, each corresponding to a depth of neighborhood exploration
\item \textbf{(3)} A fitted StandardScaler that corresponds to the concatenated, pre-PCA condensed node representation
\item \textbf{(4)} A fitted PCA that corresponds to the concatenated node representation
\end{itemize}
 
 These will be used to train a graph representation model, if that is required. In addition, this model is used for inference, and generates final node representations for use in downstream machine learning tasks.
 As noted above, Algorithm \ref{alg:nodefit} depicts only the $TrainKMeans$ function in detail. However, Lines \ref{alg:nodefit:1}-\ref{alg:nodefit:2} show that independent functions are used to train the iterative scaler, the full scaler, and the PCA. Each of these functions is nearly identical to $TrainKMeans$. For the iterative scaler, where the KMeans is fitted and stored (Line \ref{alg:nodefit:17}-\ref{alg:nodefit:18}) instead a StandardScaler is fitted and stored. This same loop is performed to train the FUllScaler and PCA, independently, but no KMeans are fitted, and the resulting representation from each iteration (Line \ref{alg:nodefit:22}) is concatenated to produce a larger embedding used for training.

\subsubsection{Node Inference for Undirected Graphs}
\begin{algorithm}[h!]
\caption{\textsf{Inferential SIR-GN Node Transform} algorithm}\label{alg:transform}
\begin{algorithmic}[1]
\Function{$NodeTransform$}{$G, ScalerPerLevel, KMeansPerLevel, FullScaler, PCA$}
    \State $NodeRep = \Call{Transform}{G, ScalerPerLevel, KMeansPerLevel}$
    \State $NodeRepScale = FullScaler.transform(NodeRep)$ \label{alg:transform:1}
    \State $NodeRepPCA = PCA.transform(NodeRepScale)$\label{alg:transform:2}
    \State \Return $NodeRepPCA$
\EndFunction
\Statex
\Function{Transform}{$G, KMeansPerLevel$}\label{alg:transform:3}
    \State $EmbList = [ ]$
    \State Initialize a matrix \textit{$Emb \in \mathbb{Z}^{|V|\times c}$} to ones \label{alg:transform:4}
    \State $Emb = Emb \times deg$ \label{alg:transform:5}
    \ForAll{$KMeans$ in  $KMeansPerLevel$}\label{alg:transform:6}
        \State $Emb1 = Norm(Emb)$
        \State $DV = \Call{VertexDescription}{Emb1, KMeans}$\label{alg:transform:8}
        \State $CR = \Call{VertexAggregation}{G, DV}$\label{alg:transform:9}
        \State $Emb = CR$
        \State append $Emb$ to $EmbList$\label{alg:transform:7}
    \EndFor
    \State \Return $EmbList$
\EndFunction
\end{algorithmic}
\end{algorithm}
Training of Inferential SIR-GN returns a model that includes pre-trained KMeans cluster centroids for each level of neighborhood exploration, along with a pre-fitted StandardScaler for each. In addition, the model has a StandardScaler pre-fitted to the final concatenated node representations, and a PCA trained on the final, normalized embeddings. With this in hand, inference is accomplished by treating the unseen target graph in the same manner as the random graphs during training with respect to the vertex description and aggregation loops (Algorithm \ref{alg:transform}). For the target graph(s), the node degree begins the process (Line \ref{alg:transform:5}). The degree vector is normalized followed by distance calculations using $depth0$ cluster centroids. These distances are converted as previously described into probabilities of membership in each cluster, and aggregation is performed, also as previously described. For simplicity, the inference algorithm treats the vector description and aggregation loops detailed in Algorithm (\ref{alg:nodefit} as functions (Lines \ref{alg:transform:8}-\ref{alg:transform:9}). This vector is normalized using the $depth0$ scaler. This is repeated to previously-chosen depth (Line \ref{alg:transform:6}), node representations from each iteration are concatenated (Line \ref{alg:transform:7}. The representation is normalized using FullScaler (Line \ref{alg:transform:1}) followed by PCA transformation (Line \ref{alg:transform:2}). Figure \ref{fig:train:2} depicts the similarity between training and inference (normalization left out for clarity), where the KMeans for each depth are used to generate each node representation using the vertex description and aggregations. The figure also depicts that the concatenated node representations are transformed at inference time using the pre-trained scaler and PCA.  The final embedding is now ready for use in downstream machine learning tasks.

\subsection{Train for Graph Representation}
If graph classification tasks will be required, a graph representation model can be trained along with the node representation model, by adding training of one additional StandardScaler and KMeans. This is accomplished after the KMeans, IncrementalScalers, FullScaler, and PCA have all completed training. A full concatenated embedding is created from random graphs, then condensed using the PCA. This set is used to fit the final Scaler, referred to in Figure \ref{fig:train:2} as the PCAScaler. One last set of random graphs is generated, embedded, normalized and condensed with the PCA as above, and finally normalized with the PCAScaler. This set of representations is used to train a final KMeans, referred to hereafter as the graph KMeans, where the cluster centroids are fit to the full concatenated, condensed embedding (Fig. \ref{fig:train:2}). Once more, of note, though the figure for simplicity only shows one random graph carried through the process, this process repeats for a large number of random graphs, and independently for each depth-level KMeans, IncrementalScaler, as well as for the FullScaler and PCA. If no graph classification is needed, this step can be removed to shorten training time.

\subsection{Inference for Graph Representation}
At inference time, generating a graph representation occurs identically to node representation until a final, PCA-transformed node embedding is generated as described for for the Node Representation algorithm. At this point, one final round of normalization is performed, using the pre-fitted PCAScaler, then we generate node descriptions from distances to the GraphKMeans cluster centroids. The node description function is identical to that described previously, but rather than a node aggregation, a graph pooling function is performed. Graph pooling in other works is often accomplished via concatenation, summing or mean-pooling node descriptions. We propose a new method, which creates a structural pseudo-adjacency matrix that is node-order agnostic. 
The Inferential SIR-GN graph pooling method creates a matrix $w\times w$ matrix, where $w$ is the number of chosen graph clusters (Algorithm \ref{alg:graphpool}, Line \ref{alg:graphpool:1}). For every edge in the network, the representation for the participating nodes is multiplied such that the resulting matrix is $w\times w$ (Lines \ref{alg:graphpool:3}-\ref{alg:graphpool:4}), and these are summed over the entire graph. This creates a node-order agnostic matrix for the graph representation:
\begin{equation} 
M_{(i,j)} = \sum_{(u,v) \in E} r_{ui}\times r_{vj},
\end{equation} 
where $i$ and $j$ represent graph clusters and $r_{u,i}$ represents the distance of node $u$ to graph cluster centroid $i$. A matrix is generated for each edge, and all edges are summed, therefore the node order has no effect on the final matrix. This generates a matrix that is unique to a graph structure, and can be linearized to form a vector of features for use in graph classification tasks.

\begin{algorithm}[h!]
\caption{\textsf{Pooling for Graph Representation} algorithm}\label{alg:graphpool}
\begin{algorithmic}[1]
\Function{$NodePool$}{$G, Emb$}
     \State Initialize a matrix \textit{$GraphRep \in \mathbb{Z}^{w\times w}$} of zeros\label{alg:graphpool:1}
     \ForAll{$u \in V$}\label{alg:graphpool:3}
        \ForAll{$nbr(u)$}
            \State $GraphRep += Emb[u].reshape((1, w))*Emb[nbr(u)].reshape((w, 1))$\label{alg:graphpool:2}
        \EndFor\label{alg:graphpool:4}
    \EndFor
\EndFunction
\end{algorithmic}
\end{algorithm}

\subsection{Inferential SIR-GN for Directed Graphs}
For directed graphs, training and inference remain the same but for small modifications to the node description and aggregation steps performed at each iteration. Directed random graphs are generated for all model training steps. In Algorithm \ref{alg:nodefit}, where the degree vector is created for each node (Lines \ref{alg:nodefit:7}-\ref{alg:nodefit:8}) we initialize two separate vectors for directed graphs, one for the node's $in-degree$ and one for its $out-degree$. These two vectors are concatenated together, then clustering and node description calculations are performed. As with undirected graphs, the node description vector after this step is equal to the probability of its membership in each cluster, based upon distance to that cluster's centroid. Algorithm \ref{alg:diragg} shows the vertex aggregation loop that is the key difference in the model for directed graphs. At the aggregation step, a separate aggregation vector is calculated for each node based upon its $in$ neighbors and its $out$ neighbors (Lines \ref{alg:diragg:1}-\ref{alg:diragg:6}). These two vectors are then concatenated together for the next iteration of clustering (Line \ref{alg:diragg:7}). This unique aggregation loop is used at every step in both the training and inference algorithms where the standard loop is used for undirected graphs.

\begin{algorithm}[h!]
\caption{\textsf{Directed Vertex Aggregation} algorithm}\label{alg:diragg}
\begin{algorithmic}[1]
\Function{DirectedVertexAggregation}{G, DV}
    \ForAll{$u \in V$} \Comment{Aggregation loop}
        \State $CR_u^{in} = [0,..,0]$\label{alg:diragg:1}
        \ForAll{$n \in nbr^{in}(u)$}\label{alg:diragg:2}
            \State $CR_u^{in} = CR_u^{in} + DV_n$\label{alg:diragg:3}
        \EndFor
        \State $CR_u^{out} = [0,..,0]$\label{alg:diragg:4}
        \ForAll{$n \in nbr^{out}(u)$}\label{alg:diragg:5}
            \State $CR_u^{out} = CR_u^{out} + DV_n$\label{alg:diragg:6}
        \EndFor
    \EndFor
    \State $PR = concatenate(CR^{in},CR^{out})$ \label{alg:diragg:7}
\EndFunction
\end{algorithmic}
\end{algorithm}

For graph representation on directed graphs, the above-described node representations are generated, followed by the standard graph training and inference methods. 

\subsection{Inferential SIR-GN for Knowledge Graphs}
Knowledge graphs differ from standard graph data in that the nodes and edges may both be heterogeneous. This requires that we are able to incorporate both node and edge labels into the node embeddings generated by Inferential SIR-GN. During training of the model, random graphs are generated along with random oneHotEncoded node or edge labels. The user selects the desired number of labels to be used for training. For strictly node-labeled data, node labels are concatenated to the existing representation at each iteration upstream of clustering and aggregation. Because a node's type will be a considered for fitting into each cluster, two nodes with identical neighborhoods (structure), but with different node type, will have different resultant embeddings.
For node-plus-edge labeled data, the model separates the node embeddings by edge type with a modified aggregation loop (Alg. \ref{alg:edgelab}, which generates a node embedding of size $c \ times \ nel$. The node's neighbors of each edge type are then summed into their designated "section" of the embedding. For this algorithm, it is important to consider the approximate number of node labels that will be encountered in the inference target graph when selecting this value for training. Models trained on a greater number of node and edge labels are tolerant to inference data with fewer labels, however, there cannot be a greater number of labels in inference than used in training. This is simply because, if the number of labels exceeds the size of the training node representation, there will be a greater number dimensionality of the initialized node vectors during inference than the dimensionality of the random data used to fit the KMeans cluster centroids. This results in an unavoidable error. To prevent this, if the model will be used on multiple different graphs, use the largest number of labels that will be encountered for training. We will show that increasing the node labels by a factor of 10 in the training phase yielded similar performance as the training with the exact number of node and edge labels in the inference data.
\begin{table}[t]
\centering
\caption{Notations used in Algorithm \ref{alg:nodefit}.}\label{tab:notations:3}
\begin{tabular}{cc}
\hline
Notation        & Description                  \\
\hline
$nnl$        & The number of node labels chosen for training  \\ 
$nel$        & The number of edge labels chosen for training \\
        \\ \hline
\end{tabular}%
\end{table}

\begin{algorithm}[h!]
\caption{\textsf{Modified Vertex Aggregation} algorithm}\label{alg:edgelab}
\begin{algorithmic}[1]
\Function{VertexAggregationLabels1}{$G, DV$}\label{alg:edgelab:7}
    \ForAll{$u \in V$} \Comment{Aggregation loop} \label{alg:edgelab:1}
        \State $CR_u = [0,..,0]$\label{alg:edgelab:2}
        \ForAll{$v \in nbr(u)$} \label{alg:edgelab:3}
            \State $el = edgeLabel(u, v)$ \label{alg:edgelab:4}
            \State $CR_{u, el} = el \times n$ \label{alg:edgelab:5}
            \State $CR_{u, el} = CR_{u, el} + DV_v$\label{alg:edgelab:6}
        \EndFor
    \EndFor
    \State \Return $CR$
\EndFunction
\end{algorithmic}
\end{algorithm}

\section{Time Complexity}
Time complexity will be described separately for training and inference, with inference described for node representation generation and graph representation generation. Training for graph representation requires the models for node representation, therefore that training is performed concurrently. 
\subsection{Training}
As described above, for model training the user selects the number of synthetic random graphs $g$ and their size $|V|=m$, the number of Kmeans clusters to be used for both the node ($c$) and graph representations ($w$), the number of final PCA components $w$(determines the final representation size), and the depth of exploration $d$.
\begin{theorem}
Given $g$ synthetic training graphs $G_t=<V,E>$ and $d$ the exploration depth, model training runs in $O(g\cdot d\cdot (f(|V|) + dc|E| + t|V|c^2+|V|c +d|V|c^2))$ where $f(|V|)$ is the time to generate random graphs, $t$ is the maximum number of MiniBatchKmeans iterations and $c$ is the number of Kmeans clusters. 
\end{theorem}

\begin{proof}
The training Algorithm separately trains a Kmeans and Scaler for each level of depth $d$, followed by a single KMeans for the graph representation, a PCA, and two additional Scalers, one each for pre-and post-PCA. Each of these models are partially fit using $g$ synthetic random graphs, each with $|V|= m$ nodes. We describe the complexity of each component of the training:

\begin{itemize}
\item \textbf{Generating Random Graphs} We will assume that our graph generator has time complexity $O(f(|V|))$. We create $g$ number of graphs for each Scaler and KMeans for each depth, for a total of $g\cdot d\cdot 2$, then an additional $g$ graphs each for the pre- and post-PCA Scalers, and $g$ graphs to train the PCA. The combined complexity to create all synthetic random graphs is  $O(f(|V|)\cdot g\cdot (3 + 2d)))$

\item \textbf{(Node Representation Aggregation)} The node aggregation step in Algorithm \ref{alg:nodefit} Lines \ref{alg:nodefit:20} - \ref{alg:nodefit:23} aggregates for each node in $V$ its neighbor vectors, for a complexity of $O(c\cdot |E|)$. This is performed for each of $g$ generated random training graphs. To calculate the complexity of node aggregation, we observe that the nodes are aggregated for each KMeans $g\cdot l$ times, where $l$ is the current level of exploration. As $l$ increases incrementally with depth of exploration, up to level $d$, the total number of occurrences of node aggregation for all depth-level Kmeans is $g\cdot (d(d+1))/2$. This is repeated exactly to train the IncrementalScaler(s). The FullScaler and PCAScaler each run $g\cdot d$ node aggregations, along with the graph KMeans and PCA training. This sums together for a total complexity for the entire training of $O(c\cdot |E|\cdot g\cdot (d^2+5d)) $

\item \textbf{(Training KMeans)} For each MiniBatchKMeans partial fit, the standard complexity is $O(t\cdot c\cdot |V|\cdot r)$ where $t$ is the number of clustering iterations, $c$ is the number of clusters, $|V|$ is the number of nodes, and $r$ is the dimension of the vector. For our algorithm, the vector size is the same as the number of clusters, making the complexity $O(t\cdot |V|\cdot c^2)$. A Kmeans is trained for each level of depth $d$. An additional Kmeans is trained for the graph representation, also on $g$ graphs, so that the total complexity for Kmeans training becomes $O((d+1)\cdot g\cdot t\cdot |V|\cdot c^2)$. 

\item \textbf{(Training Scalers)} For each level of depth $d$, an individual IncrementalScaler is trained on $g$ graphs, for a cost of $O(d\cdot g\cdot c\cdot |V|)$. The FullScaler is also trained on $g$ graphs, with $d\cdot c$ representation size, for a cost of $O(g\cdot d\cdot c\cdot |V|)$ The PCAScaler is trained on $g$ random graphs, with $p$ representation size, for a cost of $O(g\cdot p\cdot |V|)$. Notably, we have found the best performance of our model when the number of PCA components is less than 5X the number of clusters for one iteration. In fact, the models used in this work have PCA components equal to the number of clusters in one iteration, that is, $p=c$. Considering this, we can combine and simplify the complexity for  training all scalers is $O((2d+1)\cdot g\cdot |V|\cdot c)$.

\item \textbf{(Training PCA)} An IncrementalPCA is trained on a concatenated node representation of size $c\cdot d$. The standard time complexity for PCA is $O(|V|\cdot r^2)$ where $r$ is the dimension of the representation. In our case, $r=c\cdot d$, so the cost of the PCA is $O(g\cdot |V|\cdot (c\cdot d)^2)$.  
\end{itemize}

By summing all the contributing components, we calculate a training cost of $O(g((2d+3)(f(|V|))+(d^2+5d)(c\cdot |E|)+(d+1)(t\cdot |V|\cdot c^2)+(2d+1)(|V|\cdot c)+c^2d^2\cdot |V|)$. Simplifying the $d$ terms, we get a final computational complexity of $O(g\cdot d\cdot (f(|V|) + dc|E| + t|V|c^2+|V|c +d|V|c^2))$  

\end{proof}

It is important to note that the model used throughout this work was trained with $g=200$, $|V|=5000$. A depth of $d=10$ was chosen with $c=100$, and a final node representation size (PCA components) $p=100$. For the number of KMeans iterations, we used the default parameter of $t=100$ for all training. We will demonstrate that training on graphs of this size is sufficient to obtain excellent capture of a nodes structural information into a representation, even on real-world graphs of vastly larger size. 

\subsection{Inference}
\subsubsection{Node Representation}
\begin{theorem}
Assuming a single graph $G=<V,E>$ and $d$ the exploration depth, node representation inference runs in $O(d\cdot c\cdot(|V|+|E|))$ where $c$ is the number of Kmeans clusters. 
\end{theorem}

\begin{proof}
Algorithm \ref{alg:transform} line \ref{alg:transform:6} shows that for each of $d$ stored Kmeans from the training step, we complete the vertex description and aggregation steps (Lines \ref{alg:transform:8}-\ref{alg:transform:9}). These are described as functions in this description for brevity, but found in detail as loops in Algorithm \ref{alg:nodefit} Lines \ref{alg:nodefit:20}-\ref{alg:nodefit:12} (vertex description loop) and Lines \ref{alg:nodefit:13}-\ref{alg:nodefit:16} (vertex aggregation loop). The vertex description calculation is $O(c\cdot |V|)$ and the aggregation of the neighbors has a cost of $O(c\cdot |E|)$.

By summing all the contributing components, we calculate a total cost of $O(d\cdot c\cdot(|V|+|E|))$. 

\end{proof}

\subsubsection{Graph Representation}
To generate graph representations, node representations are first calculated as above. Node and graph representations can be calculated together, saving the cost described above, but here we will assume that graph representations are being calculated independently and so will include the time for node representation calculation.
\begin{theorem}
For $g$ graphs $G=<V,E>$ and an exploration $d$ the exploration depth, graph representation inference runs in $O(g\cdot d\cdot c\cdot(|V|+|E|))$ where $c$ is the number of Kmeans clusters. 
\end{theorem}

\begin{proof}
To generate the graph representation, node representations are first calculated as above. The cost of this is $O(d\cdot c\cdot(|V|+|E|))$. An additional clustering step is performed using the single graph Kmeans (no depth), for an additional $O(c\cdot(|V|)$. Following the final clustering the aggregation loop in Algorithm \ref{alg:graphpool} lines \ref{alg:graphpool:3} - \ref{alg:graphpool:4} calculates a structural pseudo-adjacency matrix, for a complexity of $O(w\cdot |E|)$. 

Summing the steps, we obtain a total cost of $O(d\cdot c\cdot(|V|+|E|) + c\cdot |V| + w\cdot |E|)$. Notably, in the model used for much of this work, we set used the same number of Kmeans clusters for node and graph clustering, such that $w=c$. This shows good performance, and is recommended, so we will for these calculations assume that $c=w$, or is within an order of magnitude. This allows the complexity to be simplified to $O(d\cdot 2\cdot c\cdot(|V|+|E|))$, which further simplifies to $O(d\cdot c\cdot(|V|+|E|))$
\end{proof}

The time complexity for directed graphs is the same as that for undirected graphs, as is that for labeled graph data.

\section{Experimental Evaluation}
\subsection{Experimental Setup}
Our experiments use an AMD Operon Processor with 2.5GHz and 256 GB of Memory.
\subsubsection{Datasets}
For our experiments, we chose a large diversity of graphs with respect to size, from hundreds to millions of nodes, and edges in the low thousands to the hundreds of millions. As these commonly used large graphs are not suited for a structural-role-based node classification, we utilize them for time-of-inference to demonstrate scalability, and assess the ability of a single pre-trained Inferential SIR-GN to capture the structural information from a large diversity of real-world data with respect to graph size, average node degree, etc. Table \ref{tab:datasets} shows a summary of the datasets used.

\begin{table*}[t] \footnotesize \center %
\begin{tabular}{ccccccccccccc|}                                                                                                                                                                   \\ \hline                                                                                  
\multicolumn{1}{c|}{Dataset}           & \multicolumn{1}{c|}{Nodes |V|}           & \multicolumn{1}{c|}{Edges |E|}          & \multicolumn{1}{c|}{Type}           & \multicolumn{1}{c|}{Experiment}       \\ \hline
\multicolumn{1}{c|}{Brazil}     & \multicolumn{1}{c|}{131} & \multicolumn{1}{c|}{1038} & \multicolumn{1}{c|}{Undirected}          & \multicolumn{1}{c|}{Regression, Node Classification, Hyperparameters} \\ \hline
\multicolumn{1}{c|}{Europe}           & \multicolumn{1}{c|}{399}           & \multicolumn{1}{c|}{5995}          & \multicolumn{1}{c|}{Undirected}           & \multicolumn{1}{c|}{Regression, Node Classification, Hyperparameters}       \\ \hline
\multicolumn{1}{c|}{USA}           & \multicolumn{1}{c|}{1190}           & \multicolumn{1}{c|}{13599}          & \multicolumn{1}{c|}{Undirected}           & \multicolumn{1}{c|}{Regression, Node Classification, Hyperparameters}       \\ \hline
\multicolumn{1}{c|}{BlogCatalog}           & \multicolumn{1}{c|}{10,312}           & \multicolumn{1}{c|}{333,983}          & \multicolumn{1}{c|}{Undirected}           & \multicolumn{1}{c|}{Regression, Scalability}       \\ \hline
\multicolumn{1}{c|}{YouTube}           & \multicolumn{1}{c|}{1,138,499}           & \multicolumn{1}{c|}{2,990,443}          & \multicolumn{1}{c|}{Undirected}           & \multicolumn{1}{c|}{Regression, Scalability}       \\ \hline
\multicolumn{1}{c|}{Orkut}           & \multicolumn{1}{c|}{3,072,441}           & \multicolumn{1}{c|}{117,185,083}          & \multicolumn{1}{c|}{Undirected}           & \multicolumn{1}{c|}{Scalability}       \\ \hline
\multicolumn{1}{c|}{Wiki}           & \multicolumn{1}{c|}{1.79M}           & \multicolumn{1}{c|}{28.5M}          & \multicolumn{1}{c|}{Directed}           & \multicolumn{1}{c|}{Scalability}       \\ \hline
\end{tabular}

\caption{Experimental Datasets}
\label{tab:datasets}
\end{table*}
 
\begin{itemize}
\item \textbf{Air-traffic networks} : This set of networks is commonly used to evaluate node structural role classification \cite{ribeiro2017struc2vec}\cite{structuralanonwalk2019}\cite{strap2019}\cite{arope2018}. For these undirected graphs, nodes represent airports and edges indicate the existence of commercial air traffic between them. We will use three different sized air-traffic networks: 
\begin{itemize}
\item \textbf{Brazilian air-traffic network}
\item \textbf{European air-traffic network}
\item \textbf{American air-traffic network}

\end{itemize}
\item \textbf{BlogCatalog} \cite{grover2016node2vec}: This is an undirected network of social relationships of users of the BlogCatalog website. 
\item \textbf{YouTube} \cite{leskovec2016snap}: This is an undirected network of video-sharing social groups of users of the YouTube website. 
\item \textbf{Wiki} \cite{leskovec2016snap}: This is a directed network of hyperlinks on Wikipedia. The ground truth communities are article categories. 
\item \textbf{Orkut} \cite{leskovec2016snap}: This is an undirected network of the Orkut online social network. Users can form and join groups, which are ground truth communities. 

Node classifications aim to predict whether a user is a member of a community, a multilabel classification task. Structural node representations are not good predictors of communities, therefore we use Orkut, YouTube, BlogCatalog, and Wiki purely as large real-world datasets to show time of inference and ability to capture node structures.
\item \textbf{MUTAG} \cite{mutag}: This is set of 188 graphs with heterogeneous nodes and edges. There are seven node types and 4 edge types. Graphs represent chemical structures, node types refer to atoms, and edge types to different chemical bonds. The ground truth for this data is each chemical's ability to cause DNA mutations.
\end{itemize}

\subsection{Scalability}

\begin{figure*}[h!]
    \includegraphics[width=\linewidth]{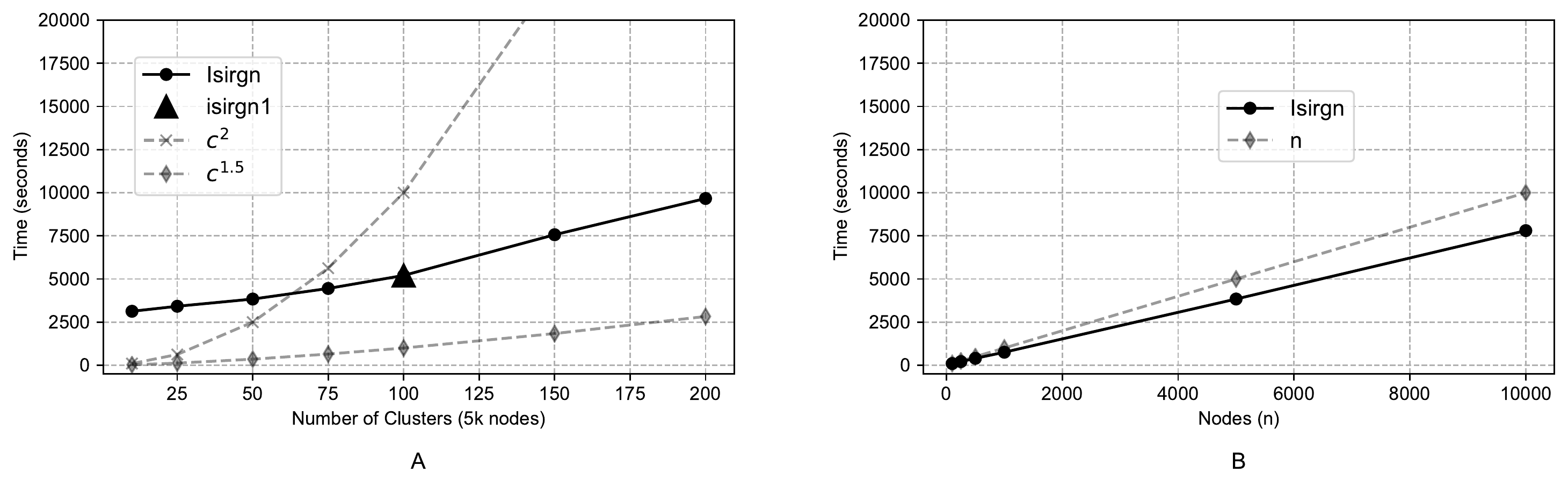}
  \caption{Effect of user-chosen hyperparameters on model training time.}
  \label{fig:time:1}
\end{figure*}

\begin{figure*}[h!]
    \includegraphics[width=\linewidth]{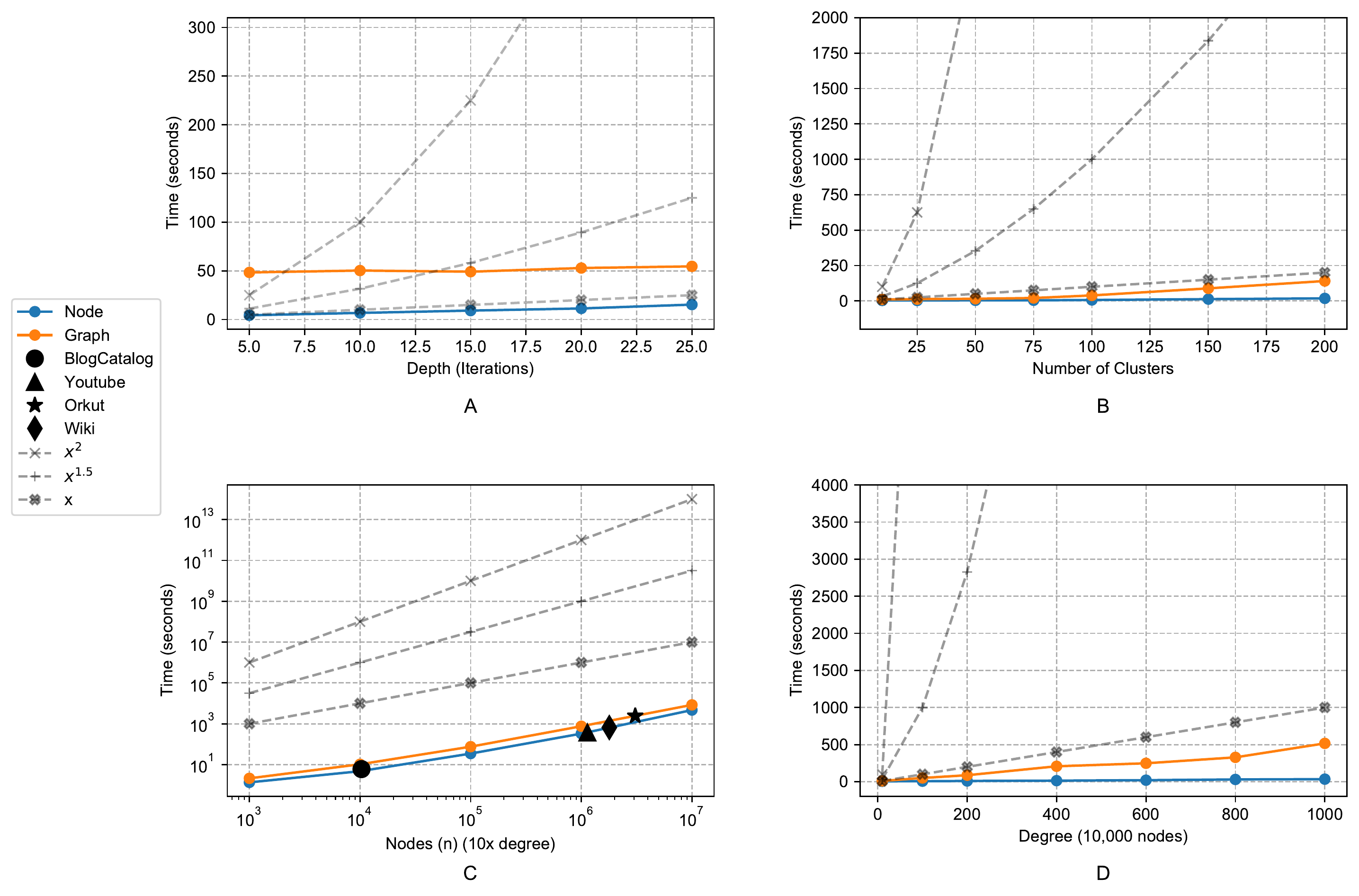}
  \caption{Effect of user-chosen hyperparameters on model inference time.}
  \label{fig:time:2}
\end{figure*}
We begin our results with a discussion on scalability, starting with an analysis of training time. Figure \ref{fig:time:1} shows the training time for a varied set of training hyperparameters discussed in the algorithm method and time complexity. For each hyperparameter tested, all others are fixed. Training time increases with the number of clusters less than quadratically (Panel B), but more than linearly, approximately $c^{1.5}$. The training time is clearly linear in the number of nodes (Panel C). Fortunately, $d<<c<<<|V|$, such that training remains highly efficient. We conduct experiments that demonstrate that depths greater than 25 are not required, nor are more than 200 clusters. Additionally, Inferential SIR-GN models can be trained on random graphs many orders of magnitude smaller than those for which inference will be conducted, and good performance still results. For no experiment in this work (other than those to present scalability) did we utilize a model with depth greater than 10, with 100 clusters, and training graphs with $|V| = 5000$. This is denoted in Figure \ref{fig:time:1} panel A as a isirgn1, marked by a triangle. \par

Figure \ref{fig:time:2} shows the time of inference for each hyperparameter described for training, plus an additional parameter for the number of edges.  We use regular random graphs to establish the time for inference on graphs of varying size with differing parameters. For training, the number of edges in the random graphs are fixed within a range that remains below $10\cdot |V|$, but this is not the case with real-world datasets. As such, we present inference time in terms of $|V|$ with $|E|$ fixed at $10\cdot |V|$, then again with $|V| = 10,000$, and increase the node degree up to $1000$. The largest graphs measured for inference time are $|V|=10^7$ and $|E| = 10^8$. At this size, generation of node representations took under two hours and graph representation inference under three hours. Importantly, we allowed the graph representation to run independently of the node inference, which means that node representations are calculated in the process of graph inference. This can be interpreted such that the \textit{total} time to obtain both node and graph representations of this size is under three hours, with the graph generation time being approximately an hour of that. \par
In \cite{nrp2020}, an excellent comparison in runtime is conducted for several relevant works, including NRP\cite{nrp2020}, AROPE\cite{arope2018}, ProNE\cite{prone2019}, NetSMF\cite{netsmf2019}, STRAP\cite{strap2019}, and VERSE\cite{verse2018}. The study was conducted on a single thread, with a processor similar to that used for our inference experiments. NRP, AROPE and ProNE showed the fastest runtime on BlogCatalog, under ten seconds. The time for a pre-trained Inferential SIR-GN model to generate node representations for BlogCatalog is approximately six seconds, in line with the fastest three models. The slower models showed exceedingly greater runtimes, with VERSE and NETSMF taking $100x$ time compared to the fastest algorithms. Importantly, we will show that a model that takes approximately 5000 seconds to train (depth 10, 100 clusters, graph-size $|V|=5000$) can be used to generate node representations that very accurately capture the structural information of a graph with over a million nodes (YouTube). This suggests that both training and inference could be performed on a massive graph in less time than taken by NetSMF, VERSE, and likely STRAP. Importantly, as the size of the target graph grows, Inferential SIR-GN \textit{gains} efficiency, as the same trained model is used and therefore training time is fixed. Training plus inference time for a graph of $|V| = 10^6$ and $|E| = 10^7$ is under 10,000 seconds (using the same hyperparameters described above). At this graph size, AROPE\cite{arope2018}, one of the most efficient models, observed a runtime $>6000s$. So, for very massive graphs, Inferential SIR-GN approaches the highest efficiency when even training time is included.

\subsection{Hyperparameter Analysis}
\begin{figure*}[h!]
    \center
    \begin{subfigure}[b]{\textwidth}{\includegraphics[width=\linewidth]{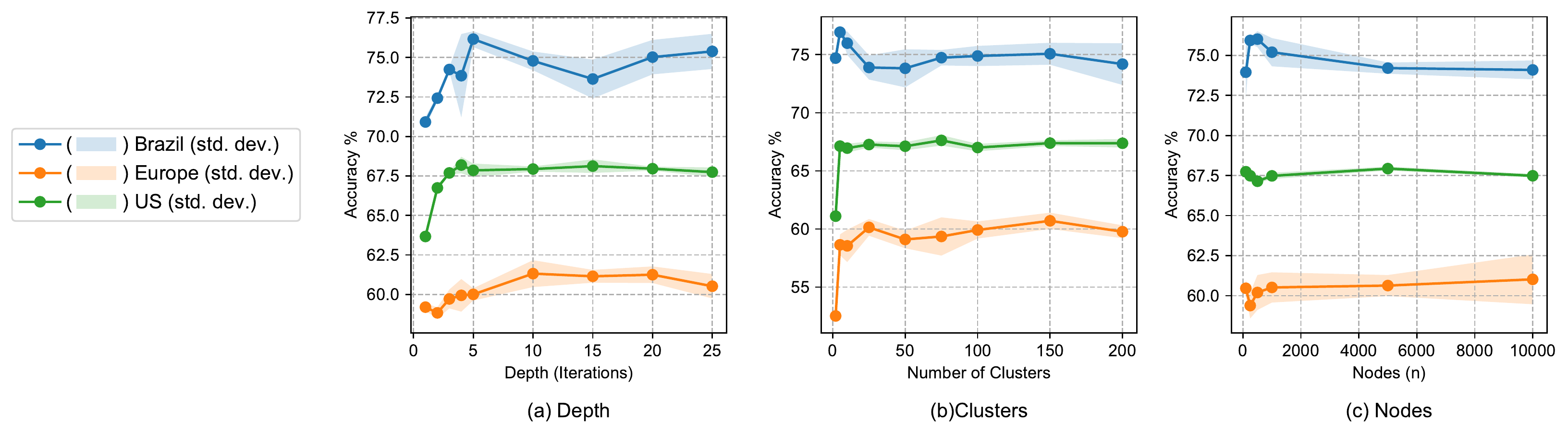}}
    \end{subfigure}
    \caption{Performance of models with different hyperparameters on three Air Traffic datasets}\label{fig:hyper}
\end{figure*}
We have examined the effect of the model hyperparameters: exploration depth, number of Kmeans clusters, and size of the training graph with respect to time complexity for model training and inference. Figure \ref{fig:hyper} presents the effect of each on model performance at node classification  tasks on three different datasets. Three independent models were generated for each hyperparameter value, with all others fixed.

Figure \ref{fig:hyper} (a) shows that lower model performance on node classification tasks is observed only at very low choices of depth, which correspond to the k-hop neighborhood incorporated into the representation for each node. Depths greater than ten show no appreciable increase in model performance, even across three datasets with approximately an order of magnitude difference between one another with respect to number of edges and nodes. 
Figure \ref{fig:hyper} (b) presents the effect of choice of clusters on the same classification task. The cluster number determines the size of the representation generated for each node per iteration of depth. As described above, each iteration is concatenated to those previous, and finally condensed using a PCA to generate the node representation. Here again, above a certain threshold (~20-50 clusters), the number of Kmeans clusters chosen does not perceptibly increase performance. 
The size of the training graph is another user-selected metric (Figure \ref{fig:hyper}c), and we generated models trained on graphs from size 100 to 10,000 nodes. Our random graph generator varies the number of edges for each graph, so each graph will have $|E|$ up to $10|V|$. Once again, above a threshold ($|V| ~ 1000$), performance does not benefit from increased training graph size. 
For the smallest real-world network (Brazil, $|V| = 131$), optimal performance may be observed when using fewer clusters, and with smaller training graphs. Experiments will show that a single Inferential SIR-GN model is more capable of capturing a node's structural information than many state-of-the art node embedding methods.

\subsection{Node Classification}
\begin{table}[h!]
    \centering
    \begin{tabular}{ l c c c c c c}
    \hline
            Model & 0.0 & 0.1 & 0.2 & 0.3 & 0.4 \\
         \hline
         \hline
         Sirgn & 100 & 83 & 71 & 62 & 57 \\
         InferSirgn & 100 & 83 & 76 & 69 & 62 \\
         GCN & 95 & 64 & 56 & 53 & 54 \\
         GIN & 100 & 80 & 75 & 65 & 59 \\
    \end{tabular}
    \caption{Accuracy on node classification for each synthetic dataset}
    \label{tab:synthnodeclass}
\end{table}

\begin{table}[h!]
    \centering
    \begin{tabular}{ l c c c c c c}
    \hline
           Model & Brazil & Europe & USA \\
         \hline
         \hline
         InferSirgn & 74.5 & 61.7 & 68.1\\ 
         STRAP & 70.71 & 65.50 & not tested\\ 
         AROPE & 64.29 & 65.00 & not tested\\ 
         VERSE & 35.71 & 42.50 & not tested\\
    \end{tabular}
    \caption{Node classification accuracy for each air traffic datasets. Note accuracies reported for STRAP, AROPE, and VERSE taken from literature reported values in \cite{strap2019}}
    \label{tab:airnodeclass}
\end{table}
We report the node classification accuracies on the Air Traffic datasets and a set of synthetic data generated by the authors for which the node's structural information is vital to classification. A graph consists of repeated substructures (one of eight possible), with added random edges (from edge percentage 0 to 40\%) between nodes, which are then re-evaluated for structural identity using the subgraph isomorphism test. We compare performance on these datasets with the original SIR-GN method, along with a Graph Convolutional Network and a Graph Isomorphism Network. Node representations generated using SIR-GN and Inferential SIR-GN are used to train an Extra Trees Classifier with 10-fold validation strategy. The number of estimators was chosen using GridSearchCV. Inferential SIR-GN most accurately classified nodes, with the GIN coming in second. GIN is expected to perform well, as it is designed to emulate the Weisfeiler-Lehman isomorphism test, and therefore classify nodes based upon structural roles. Interestingly, this performance improvement is despite the GCN, GIN, and original SIR-GN having trained on the actual datasets themselves, while Inferential SIR-GN sees them only at inference time.
On the Air Traffic datasets, we compare to the generated to maximize scalability to very large graphs. We report the literature values from experiments performed in \cite{strap2019} that compare several competing methods classified using a one vs. all logistic regression. Inferential SIR-GN (using an Extra Trees Classifier with 10-fold validataion) outperforms all other methods on the Brazil dataset, while STRAP shows the highest accuracy for the Europe Air Traffic data. 
The study in \cite{strap2019} did not evaluate the largest Air Traffic dataset, USA, where Inferential SIR-GN outperforms even the best model tested in \cite{joaristi2021sir}, including node2vec \cite{grover2016node2vec}, struc2vec \cite{ribeiro2017struc2vec}, GraphSAGE \cite{hamilton2017inductive}, and Graphwave \cite{donnat2018learning}.

\subsection{Graph Classification}
\begin{table}[h!]
    \centering
    \begin{tabular}{ l c c c c c c c}
    \hline
          Model  &Edge Labels &Node Labels & Graph Clusters & Nodes & Graphs & Depth & Accuracy\\
         \hline
         \hline
         m1 & 0 & 0 & 10 & 100 & 500 & 6 & $85.7\pm 0.82$ \\ 
         m2 & 4 & 7 & 17 & 100 & 500 & 6 &  $90.8\pm 0.80$ \\
         m3 & 10 & 10 & 20 & 100 & 500 & 6 & $88.6\pm 1.09$\\
         m4 & 15 & 15 & 25 & 100 & 500 & 6 & $88.4\pm 0.81$\\
         m5 & 20 & 20 & 30 & 100 & 500 & 6 & $89.3\pm 0.51$\\
         m6 & 25 & 25 & 35 & 100 & 500 & 6 & $89.33\pm 0.94$\\ 
         m7 & 30 & 30 & 40 & 100 & 500 & 6 & $89.2\pm 0.61$\\ 
         mutag-train & 4 & 7 & 17 & - & - & 6 & 89.3 \\ 
    \end{tabular}
    \caption{Performance of varying models on graph classification with node and edge-labeled MUTAG dataset. Note: the actual dataset has 7 node labels and 4 edge labels.}
    \label{tab:mutaginf}
\end{table}
Inferential SIR-GN generates a graph representation from the node representations by creating a structural pseudo-adjacency matrix. The node representations are clustered as described above using the Graph KMeans cluster centroids, but a graph pooling is performed at this step rather than an aggregation of the nodes' neighbors. 
The algorithm has additionally been modified to create representations from knowledge graph data, which contains nodes and edges of heterogeneous type. In order to perform machine learning tasks on these graphs, both node and edge labels must be incorporated into the structural representation for a node. The Inferential SIR-GN algorithm is modified to include node labels in the clustering and aggregation steps, rather than simply concatenate the labels to the node representation at the final step. Node representations are then separated by edge type so that each node representation is increased to $c\times el$ where $el$ is the number of edge labels. This does not increase the complexity, however, as this portion of the vector is sparse.
In order to maintain the flexibility of the algorithm to be utilized on a variety of networks, it can be trained on a larger than expected number of random node- and edge-labels, and still perform well on target graphs with fewer than training. We demostrate this flexibility in Table \ref{tab:mutaginf}, where results are displayed for graph classification on the MUTAG dataset. The table shows that with no edge labels included, the graph classification performs quite well on the dataset. Inclusion of the node and edge labels results in significant improvements to the classification accuracy, and importantly, high accuracy is still achieved with models trained on up to $4X$ the number of node labels, and $7X$ the number of edge labels as those seen in the target graph. To demonstrate this, models were first trained with the exact number of labels in the target graph, with the number of KMeans clusters at $nnl + 10$, where $nnl$ is the number of node labels. For consistency, the number of clusters that are devoted purely to the node's structural representation are maintained at ten as the number of node and edge labels is increased. We hypothesized that the PCA that finalizes the node representations would compensate for any expendable zero data created by training with a larger number of labels than will be present in the target graph. From the maintained high performance observed with even the highest number of labels tested, we gather that this assumption was correct. Importantly, this suggests that, for instance, model m7 could be used for a large variety of different datasets for graph classification tasks, and still produce excellent results. Interestingly, training on the mutag dataset itself did not yield better classification results than training on randomly generated graph data. This is unsurprising given that Inferential SIR-GN has demonstrated at least as good results as the original SIR-GN in other tasks, and may in fact improve its performance through better generalization capability. Inferential SIR-GN even outperforms the GIN reported accuracy on node classification of 89.0 \cite{xu2018powerful}, though that network was trained on the MUTAG dataset itself.

\subsection{Regression Analysis}

\begin{table*}[h!] \footnotesize \center %
\begin{tabular}{ccccccccccccc|}                                                                                                                                                                                                                                                                                                                                                                                                                                                                            \\ \cline{2-13} 
\multicolumn{1}{c|}{}           & \multicolumn{2}{c|}{PR}                                                   & \multicolumn{2}{c|}{HITS}                                                   & \multicolumn{2}{c|}{DC}                                                   & \multicolumn{2}{c|}{EC}                                                   & \multicolumn{2}{c|}{BC}                                                   & \multicolumn{2}{c|}{NCN}                              \\ \cline{2-13} 
\multicolumn{1}{c|}{}                      & \multicolumn{2}{c|}{$R^2$}         & \multicolumn{2}{c|}{$R^2$}                  & \multicolumn{2}{c|}{$R^2$}          & \multicolumn{2}{c|}{$R^2$}           & \multicolumn{2}{c|}{$R^2$}               & \multicolumn{2}{c|}{$R^2$}      \\ \hline  \hline
                                & \multicolumn{12}{c}{Brazilian air-traffic network}   \\ \hline
\multicolumn{1}{c|}{Degree}     & \multicolumn{2}{c|}{0.990}          & \multicolumn{2}{c|}{0.962}       & \multicolumn{2}{c|}{-}          & \multicolumn{2}{c|}{0.962}          & \multicolumn{2}{c|}{0.571}             & \multicolumn{2}{c|}{0.946}          \\ \hline
\multicolumn{1}{c|}{node2vec}            & \multicolumn{2}{c|}{0.111}             & \multicolumn{2}{c|}{0.203}               & \multicolumn{2}{c|}{0.137}            & \multicolumn{2}{c|}{0.203}                 & \multicolumn{2}{c|}{-0.043}           & \multicolumn{2}{c|}{0.324}        \\
\multicolumn{1}{c|}{struc2vec}          & \multicolumn{2}{c|}{0.968}                  & \multicolumn{2}{c|}{0.972}                  & \multicolumn{2}{c|}{0.975}                    & \multicolumn{2}{c|}{0.972}                  & \multicolumn{2}{c|}{0.304}            & \multicolumn{2}{c|}{0.959}          \\
\multicolumn{1}{c|}{GraphWave}           & \multicolumn{2}{c|}{0.978}             & \multicolumn{2}{c|}{0.958}                 & \multicolumn{2}{c|}{0.975}                   & \multicolumn{2}{c|}{0.956}                & \multicolumn{2}{c|}{0.383}             & \multicolumn{2}{c|}{0.920}         \\
\multicolumn{1}{c|}{GraphSAGE}       & \multicolumn{2}{c|}{0.253}                  & \multicolumn{2}{c|}{0.333}                  & \multicolumn{2}{c|}{0.275}            & \multicolumn{2}{c|}{0.333}                  & \multicolumn{2}{c|}{-0.070}          &\multicolumn{2}{c|}{0.425}        \\
\multicolumn{1}{c|}{ARGA}            & \multicolumn{2}{c|}{0.927}                    & \multicolumn{2}{c|}{0.954}                 & \multicolumn{2}{c|}{0.949}               & \multicolumn{2}{c|}{0.954}                 & \multicolumn{2}{c|}{0.606}                    &\multicolumn{2}{c|}{0.933}          \\
\multicolumn{1}{c|}{DRNE}          & \multicolumn{2}{c|}{0.991}                   & \multicolumn{2}{c|}{0.974}                 & \multicolumn{2}{c|}{0.998}                    & \multicolumn{2}{c|}{0.973}                 & \multicolumn{2}{c|}{0.569}                     & \multicolumn{2}{c|}{0.966}          \\
\multicolumn{1}{c|}{SIR-GN: GMM}   & \multicolumn{2}{c|}{0.996}          & \multicolumn{2}{c|}{0.998}                   & \multicolumn{2}{c|}{0.998}                   & \multicolumn{2}{c|}{0.998}          & \multicolumn{2}{c|}{0.618}          &\multicolumn{2}{c|}{0.965}        \\
\multicolumn{1}{c|}{SIR-GN: K-Means}  & \multicolumn{2}{c|}{0.997}  & \multicolumn{2}{c|}{\textbf{0.999}}  & \multicolumn{2}{c|}{\textbf{0.999}}   & \multicolumn{2}{c|}{\textbf{0.999}}  & \multicolumn{2}{c|}{0.795}  & \multicolumn{2}{c|}{\textbf{0.984}}\\
\multicolumn{1}{c|}{Inferential SIR-GN}  & \multicolumn{2}{c|}{\textbf{0.999}}   & \multicolumn{2}{c|}{\textbf{0.999}} & \multicolumn{2}{c|}{\textbf{0.999}}  & \multicolumn{2}{c|}{\textbf{0.999}}  & \multicolumn{2}{c|}{\textbf{0.900}}  & \multicolumn{2}{c|}{0.979} \\
\hline
\\ \hline  \hline
                                & \multicolumn{12}{c}{European air-traffic network}
                                \\ \hline
\multicolumn{1}{c|}{Degree}           & \multicolumn{2}{c|}{0.983}                    & \multicolumn{2}{c|}{0.973}          & \multicolumn{2}{c|}{-}          & \multicolumn{2}{c|}{0.973}             & \multicolumn{2}{c|}{0.522}      & \multicolumn{2}{c}{0.850} \\ \hline
\multicolumn{1}{c|}{node2vec}            & \multicolumn{2}{c|}{0.112}                  & \multicolumn{2}{c|}{0.232}          & \multicolumn{2}{c|}{0.166}           & \multicolumn{2}{c|}{0.232}           & \multicolumn{2}{c|}{0.009}          & \multicolumn{2}{c|}{0.280}                     \\
\multicolumn{1}{c|}{struc2vec}           & \multicolumn{2}{c|}{0.979}   & \multicolumn{2}{c|}{0.979}          & \multicolumn{2}{c|}{0.989}    & \multicolumn{2}{c|}{0.979}           & \multicolumn{2}{c|}{0.744}            & \multicolumn{2}{c|}{0.868}                    \\
\multicolumn{1}{c|}{GraphWave}         & \multicolumn{2}{c|}{0.856}          & \multicolumn{2}{c|}{0.968}          & \multicolumn{2}{c|}{0.873}           & \multicolumn{2}{c|}{0.968}        & \multicolumn{2}{c|}{0.808}  & \multicolumn{2}{c|}{0.879}                   \\
\multicolumn{1}{c|}{GraphSAGE}           & \multicolumn{2}{c|}{0.737}                & \multicolumn{2}{c|}{0.796}         & \multicolumn{2}{c|}{0.767}           & \multicolumn{2}{c|}{0.796}      & \multicolumn{2}{c|}{0.279}           &\multicolumn{2}{c|}{0.768}         \\
\multicolumn{1}{c|}{ARGA}          & \multicolumn{2}{c|}{0.991}          & \multicolumn{2}{c|}{0.986}    & \multicolumn{2}{c|}{0.992}    & \multicolumn{2}{c|}{0.986}           & \multicolumn{2}{c|}{0.874}        & \multicolumn{2}{c|}{0.897}         \\
\multicolumn{1}{c|}{DRNE}  & \multicolumn{2}{c|}{0.986}   & \multicolumn{2}{c|}{0.986}       & \multicolumn{2}{c|}{0.998}   & \multicolumn{2}{c|}{0.986}    & \multicolumn{2}{c|}{0.610}          &\multicolumn{2}{c|}{0.886}                    \\
\multicolumn{1}{c|}{SIR-GN: GMM} & \multicolumn{2}{c|}{0.998}          & \multicolumn{2}{c|}{\textbf{0.999}}   & \multicolumn{2}{c|}{\textbf{0.999}}   & \multicolumn{2}{c|}{\textbf{0.999}}  & \multicolumn{2}{c|}{\textbf{0.871}} &\multicolumn{2}{c|}{0.958}         \\
\multicolumn{1}{c|}{SIR-GN: K-Means} & \multicolumn{2}{c|}{0.997}  & \multicolumn{2}{c|}{\textbf{0.999}}  & \multicolumn{2}{c|}{0.998} & \multicolumn{2}{c|}{\textbf{0.999}} & \multicolumn{2}{c|}{0.853}   &\multicolumn{2}{c|}{0.955}  \\
\multicolumn{1}{c|}{Inferential SIR-GN} & \multicolumn{2}{c|}{\textbf{0.999}}  & \multicolumn{2}{c|}{\textbf{0.999}}  & \multicolumn{2}{c|}{\textbf{0.999}}  & \multicolumn{2}{c|}{\textbf{0.999}} & \multicolumn{2}{c|}{0.853}  & \multicolumn{2}{c|}{\textbf{0.965}} \\
\hline
\\ \hline  \hline
& \multicolumn{12}{c}{American air-traffic network}                  \\ \hline
\multicolumn{1}{c|}{Degree}  & \multicolumn{2}{c|}{0.889}          & \multicolumn{2}{c|}{0.932}        & \multicolumn{2}{c|}{-}    & \multicolumn{2}{c|}{0.932}     & \multicolumn{2}{c|}{0.127}                    &\multicolumn{2}{c|}{0.934}        \\ \hline
\multicolumn{1}{c|}{node2vec}     & \multicolumn{2}{c|}{0.554}           & \multicolumn{2}{c|}{0.716}    & \multicolumn{2}{c|}{0.273}    & \multicolumn{2}{c|}{0.716}           & \multicolumn{2}{c|}{0.230}            & \multicolumn{2}{c|}{0.770}        \\
\multicolumn{1}{c|}{struc2vec} & \multicolumn{2}{c|}{0.921}       & \multicolumn{2}{c|}{0.971}          & \multicolumn{2}{c|}{0.995}   & \multicolumn{2}{c|}{0.971}         & \multicolumn{2}{c|}{-0.168}            &\multicolumn{2}{c|}{0.953}         \\
\multicolumn{1}{c|}{GraphWave} & \multicolumn{2}{c|}{0.965}       & \multicolumn{2}{c|}{0.989}          & \multicolumn{2}{c|}{0.989}   & \multicolumn{2}{c|}{0.989}         & \multicolumn{2}{c|}{-0.212}                 &\multicolumn{2}{c|}{0.983}         \\
\multicolumn{1}{c|}{GraphSAGE}   & \multicolumn{2}{c|}{0.844}          & \multicolumn{2}{c|}{0.966}       & \multicolumn{2}{c|}{0.920}  & \multicolumn{2}{c|}{0.956}          & \multicolumn{2}{c|}{0.216}          & \multicolumn{2}{c|}{0.947}          \\

\multicolumn{1}{c|}{ARGA}  & \multicolumn{2}{c|}{0.976}   & \multicolumn{2}{c|}{0.975}         & \multicolumn{2}{c|}{0.989}     & \multicolumn{2}{c|}{0.975}        & \multicolumn{2}{c|}{0.628}              & \multicolumn{2}{c|}{0.950}         \\
\multicolumn{1}{c|}{DRNE}   & \multicolumn{2}{c|}{0.941}     & \multicolumn{2}{c|}{0.973}     & \multicolumn{2}{c|}{0.973}       & \multicolumn{2}{c|}{0.996}         & \multicolumn{2}{c|}{-0.387}       &  \multicolumn{2}{c|}{0.976}          \\
\multicolumn{1}{c|}{SIR-GN: GMM} & \multicolumn{2}{c|}{0.968}    & \multicolumn{2}{c|}{\textbf{0.999}}        & \multicolumn{2}{c|}{0.998}    & \multicolumn{2}{c|}{\textbf{0.999}}  & \multicolumn{2}{c|}{-0.868}          &\multicolumn{2}{c|}{0.976}          \\
\multicolumn{1}{c|}{SIR-GN: K-Means} & \multicolumn{2}{c|}{0.978}  & \multicolumn{2}{c|}{\textbf{0.999}}  & \multicolumn{2}{c|}{0.998}  & \multicolumn{2}{c|}{\textbf{0.999}} & \multicolumn{2}{c|}{0.259} & \multicolumn{2}{c|}{\textbf{0.993}} \\
\multicolumn{1}{c|}{Inferential SIR-GN} & \multicolumn{2}{c|}{\textbf{0.994}}  & \multicolumn{2}{c|}{\textbf{0.999}}& \multicolumn{2}{c|}{\textbf{0.999}} & \multicolumn{2}{c|}{\textbf{0.999}}  & \multicolumn{2}{c|}{\textbf{0.855}}  &  \multicolumn{2}{c|}{0.988} \\
\hline
\\  \hline  \hline
& \multicolumn{12}{c}{BlogCatalog}                                  \\ \hline
\multicolumn{1}{c|}{Degree}   & \multicolumn{2}{c|}{0.972}      & \multicolumn{2}{c|}{0.967}           & \multicolumn{2}{c|}{-}    & \multicolumn{2}{c|}{0.967}        & \multicolumn{2}{c|}{0.834}                  & \multicolumn{2}{c|}{0.912}        \\ \hline
\multicolumn{1}{c|}{node2vec}    & \multicolumn{2}{c|}{0.736}         & \multicolumn{2}{c|}{0.819}      & \multicolumn{2}{c|}{0.747}   & \multicolumn{2}{c|}{0.819}          & \multicolumn{2}{c|}{0.409}                 &\multicolumn{2}{c|}{0.838}        \\
\multicolumn{1}{c|}{struc2vec}   & \multicolumn{2}{c|}{0.974}          & \multicolumn{2}{c|}{0.979}        & \multicolumn{2}{c|}{0.980}    & \multicolumn{2}{c|}{0.979}           & \multicolumn{2}{c|}{0.866}               &\multicolumn{2}{c|}{0.926}        \\
\multicolumn{1}{c|}{GraphWave}& \multicolumn{2}{c|}{0.980}      & \multicolumn{2}{c|}{0.994}            & \multicolumn{2}{c|}{0.936}     & \multicolumn{2}{c|}{0.981} & \multicolumn{2}{c|}{0.882}            & \multicolumn{2}{c|}{0.960}      \\
\multicolumn{1}{c|}{GraphSAGE}  & \multicolumn{2}{c|}{0.292}        & \multicolumn{2}{c|}{0.446}           & \multicolumn{2}{c|}{0.311}     & \multicolumn{2}{c|}{0.446}       &  \multicolumn{2}{c|}{0.058}              & \multicolumn{2}{c|}{0.549}         \\
\multicolumn{1}{c|}{ARGA}    & \multicolumn{2}{c|}{0.978}      & \multicolumn{2}{c|}{0.986}                & \multicolumn{2}{c|}{0.983}    & \multicolumn{2}{c|}{0.986}           & \multicolumn{2}{c|}{0.916}    &\multicolumn{2}{c|}{0.949}    \\
\multicolumn{1}{c|}{DRNE} & \multicolumn{2}{c|}{0.542}   & \multicolumn{2}{c|}{0.734}     & \multicolumn{2}{c|}{0.567}   & \multicolumn{2}{c|}{0.734}  & \multicolumn{2}{c|}{0.242}           &\multicolumn{2}{c|}{0.883}         \\
\multicolumn{1}{c|}{SIR-GN: GMM} & \multicolumn{2}{c|}{0.981}  & \multicolumn{2}{c|}{0.993}    & \multicolumn{2}{c|}{0.959}    & \multicolumn{2}{c|}{0.993}   & \multicolumn{2}{c|}{0.837}            & \multicolumn{2}{c|}{\textbf{0.975}}          \\
\multicolumn{1}{c|}{SIR-GN: K-Means}  & \multicolumn{2}{c|}{0.985}     & \multicolumn{2}{c|}{0.997}   & \multicolumn{2}{c|}{0.961}  & \multicolumn{2}{c|}{0.996}  & \multicolumn{2}{c|}{0.879}  &\multicolumn{2}{c|}{0.968} \\ 
\multicolumn{1}{c|}{Inferential SIR-GN}   & \multicolumn{2}{c|}{\textbf{0.999}}         & \multicolumn{2}{c|}{\textbf{0.999}}  & \multicolumn{2}{c|}{\textbf{0.999}}  & \multicolumn{2}{c|}{\textbf{0.999}} & \multicolumn{2}{c|}{\textbf{0.964}}  &\multicolumn{2}{c|}{0.970}  \\ \hline
\\  \hline  \hline
& \multicolumn{12}{c}{YouTube}                                  \\ \hline
\multicolumn{1}{c|}{Inferential SIR-GN}   & \multicolumn{2}{c|}{0.999}        & \multicolumn{2}{c|}{-}       & \multicolumn{2}{c|}{0.999} & \multicolumn{2}{c|}{-}          & \multicolumn{2}{c|}{-}          & \multicolumn{2}{c|}{-}     \\ \hline


\end{tabular}

\caption{Results of the regression experiments on network metrics: PageRank (PR), HITS Authority/Hub (HITS), Degree Centrality (DC), Eigenvector Centrality (EC), Betweenness Centrality (BC), Node Clique Number (NCN). Values in bold correspond to the best results obtained for each metric.}
\label{tab:regression}
\end{table*}

In \cite{joaristi2021sir}, authors propose to understand the ability of several node embedding methods to capture the structural information of the nodes in a network, including the first-generation SIR-GN. Briefly, node representations are generated using each model, then those embeddings are tested for their ability to predict graph metrics associated with node structures, including: PageRank, Hits@Authority, Degree Centrality, Eigenvector Centrality, Betweenness Centrality, and Node Clique Number, each described above. The outcome of the experiment in \cite{joaristi2021sir} is included for reference (Table \ref{tab:regression}, and we compare Inferential SIR-GN with the existing data with respect to the $R^2$. The original SIR-GN outperformed many existing state-of-the-art node embedding methods at node structural role representation on every dataset tested. An important difference to note is that SIR-GN uses the target graph to train the clustering KMeans, and so the model has been retrained for each unique dataset. Inferential SIR-GN results are all generated from a single pre-trained model. Another difference is the size of the representation: for consistency with the node classification experiments, we are using a representation size of 100 instead of 20. This generates a more universal model of Inferential SIR-GN that will be used for multiple datasets, some of which are very large compared to the training graphs. 
Inferential SIR-GN uses the same internal node clustering and aggregation as does the original model, and as such, we hypothesize that it will share the excellent capacity to capture node structural information. As described above, Inferential SIR-GN improves the ability of the original not only in it's inferential capability, but in its structural representation, by capturing the evolution of each node's neighborhood exploration. Therefore, it is unsurprising that our model performs at least as well as the non-inferential model at this task.
To further demonstrate the flexibility of a single model, we have included the YouTube dataset, a network with over one million nodes. Inferential SIR-GN, though trained on random graphs with 5000 nodes and with a depth of 10, and 100 KMeans clusters, still fully captures the structural information of the nodes in the network with respect to the PageRank and Degree Centrality measures. Additional metrics could not be tested in feasible time, as calculating Node Clique Number and Betweenness centrality is time prohibitive on such a massive graph. Additionally, nonlinear kernels required for regression experiments with Hits and Eigenvector Centrality were also exceedingly time costly on a network of this size. However, the data collected suggests that one model can be trained, then successfully used to generate structural node representations for any required graph. We have tested the same model on a diversity of network sizes, ranging from ~100 nodes to over 1 million nodes and millions of edges. We propose that, with an inference time equally as efficient as the current fastest algorithms for massive graphs, even greater total gains in efficiency are obtained by reuse of the same model for all tasks.

\clearpage
\printbibliography

\end{document}